\PassOptionsToPackage{breaklinks,hidelinks}{hyperref}
\documentclass[11pt]{article}
\usepackage[preprint]{acl}
\usepackage{times}
\usepackage{latexsym}
\usepackage[T1]{fontenc}
\usepackage[utf8]{inputenc}
\usepackage{microtype}
\usepackage{inconsolata}
\usepackage{amsmath,amssymb,amsthm}
\usepackage{booktabs}
\usepackage{graphicx}
\usepackage{algorithm}
\usepackage{algorithmic}
\usepackage{xcolor}
\usepackage{subcaption}
\usepackage{tikz}
\usetikzlibrary{arrows.meta,positioning,fit}

\usepackage{pgfplots}
\pgfplotsset{compat=1.17}
\definecolor{cblue}{RGB}{224,236,248}
\definecolor{cgreen}{RGB}{226,240,217}
\definecolor{corange}{RGB}{252,229,205}
\definecolor{cpurple}{RGB}{232,222,248}
\definecolor{cgray}{RGB}{242,242,242}
\definecolor{cline}{RGB}{90,90,90}

\makeatletter
\@ifpackageloaded{lineno}{%
  \setlength\linenumbersep{12pt}%
}{}
\makeatother
\newtheorem{theorem}{Theorem}

\newtheorem{proposition}[theorem]{Proposition}

\newtheorem{definition}[theorem]{Definition}

\newcommand{\eps}{\varepsilon}
\newcommand{\cL}{\mathcal{L}}
\newcommand{\cQ}{\mathcal{Q}}
\newcommand{\cC}{\mathcal{C}}

\newcommand{\R}{\mathbb{R}}
\newcommand{\E}{\mathbb{E}}

\title{Canonicalized Stable-List Replay for Private Federated Continual Learning over Language-Model Embeddings}

\author{
\textbf{Ibne Farabi Shihab}\textsuperscript{1}%
\thanks{Corresponding author: \texttt{ishihab@iastate.edu}.}
\and
\textbf{Abu Sa-Adat Mohamed Moon-Im Al Ahsan}\textsuperscript{2}
\and
\textbf{Anuj Sharma}\textsuperscript{3}
\\[2pt]
\textsuperscript{1}Department of Computer Science, Iowa State University \\
\textsuperscript{2}Department of Computer Science \& Engineering, BRAC University \\
\textsuperscript{3}Department of Civil, Construction \& Environmental Engineering, Iowa State University \\
\texttt{ishihab@iastate.edu}, \texttt{abu.sa.adat.mohamed.moon.im.al.ahshan@g.bracu.ac.bd}
}

\begin{document}
\maketitle

\begin{abstract}
Federated continual learning (FCL) lets distributed clients adapt language-model heads to evolving NLP tasks without sharing raw text. Under user-level differential privacy (DP), replay-based continual learning faces a structural obstacle: 
clients can release only small noisy lists of candidate replay summaries, and those lists are unordered across clients. We introduce Canonicalized Stable-List Replay (CSLR), where clients privately produce candidate replay distributions over a shared sentence-embedding space and the server aligns them using signatures induced by public anchor sentences. The anchors provide identifiability for aggregation rather than additional replay data. We prove that, under an observable anchor-signature margin, $O(\log(N/\eta)/p)$ anchors distinguish $N$ candidate list elements with probability at least $1-\eta$, and we give a scoped anchorless non-identifiability result for unordered-label oracle models. Across five seeds on continual classification, NER, and dialogue benchmarks, CSLR improves the final average task metric by 3.9--5.6 points over the strongest non-CSLR DP baseline at $\eps=4$ under the reported replay-release budget, while also outperforming Hungarian and optimal-transport matchers. The formal privacy guarantee covers replay release; end-to-end private training additionally requires composition with a private optimizer for task-head updates.
\end{abstract}

%---------------------------------------------------------------
\section{Introduction}
\label{sec:intro}
%---------------------------------------------------------------

Language models deployed across distributed clients---such as hospitals, mobile devices, and organizations handling sensitive documents---must continually adapt to new tasks and domains. Federated learning (FL) enables collaborative training without centralizing text, while continual learning (CL) addresses evolving task streams; their intersection, \emph{federated continual learning} (FCL), has therefore attracted growing attention \citep{yoon2021federated,qi2023better,babakniya2023dont}.

The central difficulty in FCL is catastrophic forgetting: as clients learn new tasks, performance on earlier tasks degrades. Replay methods mitigate this by rehearsing past data \citep{rolnick2019experience}, but raw-text replay conflicts with FL privacy goals, and even model-derived statistics can leak information. Existing private approaches use differentially private data sharing \citep{yoon2021federated}, generative replay \citep{qi2023better}, or diffusion-based synthesis \citep{liang2024diffusion}; however, they may weaken privacy guarantees, degrade under DP noise, or struggle with heterogeneous client distributions.

We pursue a different replay interface. Instead of requiring each client to release one replay distribution---a brittle target under user-level DP noise---we allow a small \emph{list} of candidate distributions in sentence-embedding space. This follows stable list decoding in DP density estimation \citep{afzali2024agnostic,ghazi2021power}, where producing a list can reduce sample complexity. In FCL, this lets clients privately report useful replay candidates from limited local data, while avoiding the stronger noise burden of a single estimate.

Lists, however, introduce a symmetry problem. If Client A outputs $\{q_1^A,q_2^A,q_3^A\}$ and Client B outputs $\{q_1^B,q_2^B,q_3^B\}$, their local indices need not correspond: mixture components are permutation-invariant, so naive aggregation is undefined. The server must first determine which candidate distributions represent the same replay mode.

We resolve this with anchor canonicalization. Public anchor sentences, drawn for example from Wikipedia, induce signatures on list elements and thereby break permutation symmetry. Unlike FedTA \citep{yang2024fedta}, where anchors improve representation quality, our anchors serve \emph{identifiability}: they make aggregation well-defined before replay quality can even be evaluated. We introduce CSLR, a list-based replay interface for private FCL over frozen language-model embeddings, together with an anchor-based procedure for aligning unordered client lists. We prove a canonicalization guarantee under an explicit anchor-visible margin assumption (Theorem~\ref{thm:canon}), give an anchorless non-identifiability result in an unordered-label oracle model (Proposition~\ref{prop:barrier}), and provide replay-release privacy accounting in which local list release and noisy aggregation consume privacy budget, while anchor matching, PSD projection, and replay sampling are post-processing. End-to-end user-level DP additionally requires composing this replay-release budget with a DP optimizer for task-head updates.

%----------------------------------------------------------------------
\section{Related Work}
\label{sec:related}
%----------------------------------------------------------------------

Federated continual learning combines FL's privacy constraints with CL's temporal dynamics. FedWeIT \citep{yoon2021federated} decomposes weights into global and task-specific parameters; FedCIL \citep{qi2023better} uses generative replay with model consolidation; and diffusion-based replay has also been studied for continual federated learning \citep{mei2024diffusion}. AF-FCL \citep{afcl2025} introduces selective replay, and \citet{fcl_survey2025} provide a recent survey. These works mitigate forgetting, but they do not address the identifiability problem created when clients release unordered private lists of replay distributions.

Several related systems transfer compact representations rather than raw examples. Powder \citep{piao2024powder} studies prompt-based dual knowledge transfer in FCL, while MOON \citep{moon2021} uses contrastive representation alignment to reduce client drift in non-IID FL. CSLR is complementary: it also uses alignment, but aligns replay \emph{distributions} rather than model representations. Similarly, data-diet methods \citep{data_diet2021} show that selecting informative examples can improve training efficiency; CSLR instead releases a list of high-utility private replay candidates and then canonicalizes them across clients.

On privacy, DP-SGD \citep{abadi2016deep} and federated DP training \citep{mcmahan2018learning} are standard, with NLP applications in classification \citep{li2022large}, NER \citep{lyu2020differentially}, and generation \citep{yue2023synthetic}. Parameter-efficient private tuning, including private prompt-tuning analyses such as Flocks of Stochastic Parrots \citep{duan2023flocks}, is also relevant because CSLR adapts lightweight heads over frozen representations rather than full encoders. DP-FCL evaluations \citep{ouyang2023dpfcl} further show that privacy mechanisms can interact strongly with forgetting and performance. Our replay-release interface builds on list global stability \citep{ghazi2021power} and agnostic private density estimation for Gaussian mixtures \citep{afzali2024agnostic}, but adds the federated canonicalization layer needed for cross-client replay aggregation. Anchor-based methods also appear in FL, but with different goals. FedTA \citep{yang2024fedta} uses tail anchors for representation quality, and FedProto \citep{tan2022fedproto} aggregates class prototypes. Our anchors instead provide identifiability: they break permutation symmetry so list aggregation is well-defined before replay quality is optimized.

%----------------------------------------------------------------------
\section{Problem Setting}
\label{sec:setting}
%----------------------------------------------------------------------

We consider a federated continual learning model in which time proceeds in rounds $t = 1, \ldots, T$. At round $t$, a subset $\cC_t$ of clients participates, and each client $i \in \cC_t$ holds current text data $D_{i,t}$ from an NLP task that may drift over time---for instance, a new classification domain, a new NER schema, or a new dialogue intent. The objective is to minimize average risk across all NLP tasks seen so far, $\frac{1}{T} \sum_{t=1}^T R_t(\theta)$, where $R_t$ is the risk on task $t$ and $\theta$ denotes the model parameters. Standard FCL metrics include average accuracy and backward transfer.

We target user-level $(\eps, \delta)$-DP, in which neighboring datasets differ in the entire contribution of one client across the training horizon. This is stronger than record-level DP and is appropriate when each client is a user whose full participation must be protected. We assume access to a pretrained sentence encoder $\phi : \text{text} \to \R^d$ (for example, a frozen Sentence-BERT \citep{reimers2019sentence} or any publicly pretrained encoder). Operating in embedding space makes density modeling tractable while letting the task head be trained privately on top of representations that are shared but unchanged. Throughout, ``replay'' refers to rehearsing embeddings and their associated targets, not raw text.

We distinguish two privacy scopes. The first is the \emph{replay-release} scope: clients release private list candidates and noisy replay summaries from which the server constructs $Q_t$. This is the mechanism analyzed in Theorem~\ref{thm:privacy} and used for the reported privacy--utility experiments. The second is \emph{end-to-end federated training}: if task-head weights, gradients, or model updates are transmitted during training, those updates must also be privatized, for example with DP-FedAvg or DP-SGD. The end-to-end budget is therefore the sequential composition of replay release and private model-update training. Unless explicitly stated otherwise, $\eps$ in our experimental tables denotes the replay-release budget.

%----------------------------------------------------------------------
\section{Canonicalized Stable-List Replay}
\label{sec:method}
%----------------------------------------------------------------------

CSLR is organized as five steps that map cleanly onto a single communication round (Figure~\ref{fig:pipeline}, Algorithm~\ref{alg:cslr}): clients embed their data and run a DP list density learner; they compute signatures of each candidate against a shared anchor set; the server clusters signatures into canonical groups; per-group sufficient statistics and replay-target summaries are aggregated under DP; and clients train on a mixture of current data and replay drawn from the global generator $Q_t$. We describe each step in turn, emphasizing the design choices that make the pipeline well defined under permutation symmetry and user-level DP.

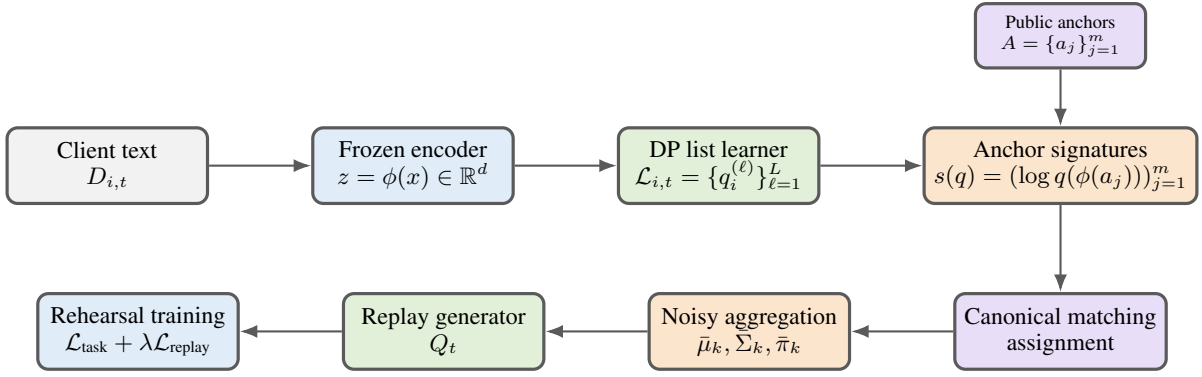
\begin{figure*}[t]
 \centering
 \begin{tikzpicture}[
   node distance=1.35cm,
   block/.style={
     draw=cline,
     rounded corners,
     very thick,
     align=center,
     minimum height=1.0cm,
     minimum width=2.65cm,
     font=\small
   },
   smallblock/.style={
     draw=cline,
     rounded corners,
     very thick,
     align=center,
     minimum height=0.85cm,
     minimum width=2.25cm,
     font=\scriptsize
   },
   arrow/.style={-Latex, thick, draw=cline}
 ]
   \node[block, fill=cgray]   (data)   {Client text\\$D_{i,t}$};
   \node[block, fill=cblue, right=of data]   (embed)  {Frozen encoder\\$z=\phi(x)\in\mathbb R^d$};
   \node[block, fill=cgreen, right=of embed] (list)   {DP list learner\\$\mathcal L_{i,t}=\{q_i^{(\ell)}\}_{\ell=1}^L$};
   \node[block, fill=corange, right=of list] (sig)    {Anchor signatures\\$s(q)=(\log q(\phi(a_j)))_{j=1}^m$};
   \node[block, fill=cpurple, below=1.15cm of sig] (match) {Canonical matching\\assignment};
   \node[block, fill=corange, left=of match]  (agg)    {Noisy aggregation\\$\bar\mu_k,\bar\Sigma_k,\bar\pi_k$};
   \node[block, fill=cgreen, left=of agg]     (replay) {Replay generator\\$Q_t$};
   \node[block, fill=cblue, left=of replay]   (train)  {Rehearsal training\\$\mathcal L_{\text{task}}+\lambda\mathcal L_{\text{replay}}$};
   \node[smallblock, fill=cpurple, above=0.75cm of sig] (anchors) {Public anchors\\$A=\{a_j\}_{j=1}^m$};

   \draw[arrow] (data) -- (embed);
   \draw[arrow] (embed) -- (list);
   \draw[arrow] (list) -- (sig);
   \draw[arrow] (anchors) -- (sig);
   \draw[arrow] (sig) -- (match);
   \draw[arrow] (match) -- (agg);
   \draw[arrow] (agg) -- (replay);
   \draw[arrow] (replay) -- (train);
 \end{tikzpicture}
 \caption{CSLR pipeline. Clients release private list candidates in a shared embedding space. Public anchors convert unordered candidates into signatures, enabling canonical matching before noisy aggregation into a replay generator.}
 \label{fig:pipeline}
\end{figure*}

\begin{algorithm}[t]
\caption{CSLR: one communication round}
\label{alg:cslr}
\small
\begin{algorithmic}[1]
\REQUIRE participating clients $\cC_t$; public anchors $A=\{a_j\}_{j=1}^m$; frozen encoder $\phi$; list size $L$; canonical modes $K$; clipping bounds $B_\mu,B_M,B_\pi,B_y$; Gaussian noise scales $\sigma_\mu,\sigma_M,\sigma_\pi,\sigma_y$.
\ENSURE private replay generator $Q_t$.
\FOR{each client $i\in\cC_t$}
  \STATE Embed local examples $z=\phi(x)$ and fit a clipped DP list learner to obtain $\cL_{i,t}=\{q_{i,t}^{(\ell)}\}_{\ell=1}^L$.
  \STATE Attach the corresponding conditional replay targets (\S\ref{subsec:replay_targets}).
  \STATE Compute signatures $s(q_{i,t}^{(\ell)})=(\log q_{i,t}^{(\ell)}(\phi(a_j)))_{j=1}^m$ for all $\ell$.
\ENDFOR
\STATE Cluster signatures into canonical groups $1,\ldots,K$ or solve a minimum-cost matching to running prototypes.
\FOR{each canonical group $k$}
  \STATE Clip candidate means, second moments, mixture weights, and replay-target summaries; aggregate them through secure aggregation and add Gaussian noise.
  \STATE Project the noisy covariance estimate onto the positive semidefinite cone with eigenvalue floor $\lambda_{\min}$.
\ENDFOR
\STATE Project noisy mixture weights $(\bar\pi_1,\ldots,\bar\pi_K)$ onto the probability simplex $\Delta_K$, and denote the result by $(\hat\pi_1,\ldots,\hat\pi_K)$.
\STATE Return the mixture $Q_t=\sum_{k=1}^K \hat\pi_k\,\mathcal N(\bar\mu_k,\bar\Sigma_k)$ with associated replay targets.
\end{algorithmic}
\end{algorithm}

\subsection{Local DP List of Replay Distributions}

Each participating client embeds its current text data $\{x_j\}_{j=1}^{n_i}$ into $\{\phi(x_j)\}$ and runs a DP list-decodable density estimator to produce a small list of candidate replay distributions $\cL_{i,t} = \{q_{i,t}^{(1)}, \ldots, q_{i,t}^{(L)}\}$, where each $q_{i,t}^{(\ell)}$ is a Gaussian (or Gaussian mixture component) in $\R^d$. Following \citet{afzali2024agnostic}, the list output satisfies a stability property: with high probability, at least one $q^{(\ell)}$ is close to the client's true embedding distribution. The privacy--utility benefit of listing arises because, with a single output, DP noise must mask the worst-case sensitivity, whereas with a list the noise can be distributed across candidates.

\subsection{Replay Target Specifications across NLP Tasks}
\label{subsec:replay_targets}

For the task head to be trained continuously, each component $q^{(\ell)}$ must carry a corresponding replay target, and the target itself must satisfy a bounded sensitivity so that aggregation remains compatible with DP. The instantiation differs by task family. In text classification (AG-News, Amazon), each component is mapped to a discrete task label distribution or a soft-label logit vector, derived by averaging local logits assigned to the component and clipped to $\|g\|_2 \le B_g$ before secure aggregation. In sequential NER (OntoNotes-Domain), candidate components correspond to key entity-type contexts, and replay targets are encoded as clipped token-class frequency vectors attached to each component. In continual dialogue (MultiWOZ-Intent), replay structures store multi-intent activation masks; clients clip the active intent multi-hot vectors and aggregate them conditionally alongside the component's structural parameter estimation.

Across all cases the replay target vector $y$ (or logit vector $g$) is clipped to $\|y\|_2\le B_y$ before aggregation. The released replay distribution is therefore a joint mixture over embeddings and conditional targets, $Q_t(z,y)=\sum_k \bar\pi_k\mathcal N(z;\bar\mu_k,\bar\Sigma_k)\bar r_k(y)$, where $\bar r_k$ is the noisy clipped target summary for canonical group $k$. This avoids replaying unlabeled embeddings into a supervised head.

\subsection{Anchor Set and Canonical Signatures}

The server maintains a small anchor set $A = \{a_1, \ldots, a_m\}$ of public sentences, sampled once from a corpus such as Wikipedia and embedded as $\{\phi(a_j)\}_{j=1}^m$. For each candidate distribution $q$ in a client's list, the \emph{anchor signature} is the vector of log-densities at the anchor embeddings,
\begin{equation}
 s(q) = \big(\log q(\phi(a_1)), \ldots, \log q(\phi(a_m))\big) \in \R^m.
\end{equation}
Because each candidate distribution $q$ is produced by the local DP list learner, the corresponding signature $s(q)$ is a deterministic function of a private release. In the main protocol, anchor-signature computation, signature clustering, canonical ID assignment, PSD projection, and replay sampling are therefore post-processing operations and do not consume additional privacy budget. This post-processing claim does not apply to variants in which signatures are computed directly from non-private local sufficient statistics; such variants must privatize the signature release itself or replace the matching step with a DP clustering primitive. Our accounting below treats the local list release and the noisy replay aggregation as the privacy-consuming replay-release components.

\subsection{Canonicalization via Signature Clustering}

The server collects all signatures $\{s(q_{i,t}^{(\ell)})\}$ and clusters them into $K$ canonical groups using a deterministic stable clustering rule or a minimum-cost assignment to running prototypes. Each cluster corresponds to a global replay mode---a coherent replay distribution that multiple clients contribute to. The clustering step assigns canonical IDs $k \in [K]$ to list elements, resolving the permutation ambiguity, and each client $i$'s list is reordered into a canonical list $\tilde{\cL}_{i,t}$ aligned with the global IDs. Because the signatures in our main protocol are post-processing of DP list outputs, the matching step is also post-processing; if raw signatures from non-DP statistics are used instead, this line must be replaced by a DP clustering primitive and included in the accountant.

\subsection{DP Aggregation and Selection}

With canonicalized lists, the server aggregates per-mode. In the parametric variant, candidates are Gaussian, $q^{(\ell)} = \mathcal{N}(\mu^{(\ell)}, \Sigma^{(\ell)})$, and the server aggregates clipped sufficient statistics through secure aggregation plus Gaussian noise per canonical ID:
\begin{equation}
\begin{aligned}
\bar{\mu}_k
&=
\frac{1}{|\cC_t|}
\sum_{i\in\cC_t}\tilde{\mu}_{i,k}
+
\mathcal{N}(0,\sigma_\mu^2 I),\\
\bar{M}_k
&=
\frac{1}{|\cC_t|}
\sum_{i\in\cC_t}\tilde{M}_{i,k}
+
\mathcal{N}(0,\sigma_M^2 I).
\end{aligned}
\end{equation}
Mixture weights and replay-target summaries are handled analogously: each client clips $\pi_{i,k}$ and $y_{i,k}$ before aggregation, and the server releases noisy $\bar\pi_k$ and $\bar r_k$ with Gaussian noise calibrated to their clipping sensitivities. Because Gaussian noise can make mixture weights negative or make them fail to sum to one, the noisy vector $\bar\pi=(\bar\pi_1,\ldots,\bar\pi_K)$ is projected onto the probability simplex \[ \Delta_K=\{\pi\in\mathbb R_+^K:\sum_{k=1}^K\pi_k=1\}, \] and we use the projected weights $\hat\pi$ when constructing $Q_t$. This projection is deterministic post-processing and therefore does not affect the replay-release DP accountant. An alternative is stability-based selection in the sense of \citet{ghazi2021power}: a DP-stable selection rule chooses a representative from the union of canonicalized lists without a server validation set. In implementation we aggregate second moments rather than raw covariance matrices: each client clips $\mu$ and $M=\E[zz^\top]$, the server adds Gaussian noise to the secure aggregate, and the covariance is recovered as $\Sigma=M-\mu\mu^\top$. The result is symmetrized and projected onto the positive semidefinite cone by flooring eigenvalues at $\lambda_{\min}>0$, guaranteeing a valid Gaussian component after noise addition. The aggregated per-mode distributions, simplex-projected weights, and target summaries are concatenated into the global replay generator $Q_t$.

\subsection{Replay-Augmented Training}

At each round, the server samples replay embeddings $z \sim Q_t$, distributes them to participating clients, and clients train on a mixture of current-task data and the replayed embeddings:

\begin{equation}
 \mathcal{L}_{\text{total}}
 =
 \mathcal{L}_{\text{task}}(D_{i,t}; \theta)
 +
 \lambda \mathcal{L}_{\text{replay}}(Q_t; \theta).
\end{equation}

where $\mathcal{L}_{\text{replay}}$ is a distillation or contrastive loss in embedding space. The replay term anchors the head to past distributional support without requiring any client to retain raw past text.

%----------------------------------------------------------------------
\section{Theoretical Analysis}
\label{sec:theory}
%----------------------------------------------------------------------

Anchors help only when distinct candidates differ on regions observed by the public anchor distribution. We formalize this as anchor-visible separation, bound the required number of anchors, and delimit the regime where anchorless matching is insufficient.

\begin{definition}[Anchor-visible separation]
\label{def:anchorvisible}
Let $\rho$ be the public anchor distribution over embedded anchor points $u=\phi(a)$. A family $\cQ$ is $(p,\gamma)$ anchor-visible if for every pair $q,q'\in\cQ$ that should receive different canonical IDs,
\begin{equation}
 \rho\!\left(\left\{u: |\log q(u)-\log q'(u)|\ge \gamma\right\}\right) \ge p .
\end{equation}
The parameter $p$ measures how often public anchors land in regions where the pair is distinguishable, and $\gamma$ is the signature margin at such anchors.
\end{definition}

\begin{theorem}[Anchor sufficiency under observable signature margins]
\label{thm:canon}
Consider $N$ candidate replay distributions appearing across all clients and rounds, where $N=\sum_{t=1}^T |\cC_t|L$. Suppose every pair of candidates that should receive different canonical IDs is $(p,\gamma)$ anchor-visible with respect to the public anchor distribution $\rho$. Draw $m$ anchors independently from $\rho$ and define $s(q)=(\log q(\phi(a_1)),\ldots,\log q(\phi(a_m)))$. If
\begin{equation}
 m \ge \frac{1}{p}\log\frac{N^2}{\eta},
\end{equation}
then with probability at least $1-\eta$ every distinct pair has at least one anchor coordinate $j$ such that $|s_j(q)-s_j(q')|\ge \gamma$. Consequently, any deterministic matching rule that separates signatures at margin $\gamma/2$ uniquely distinguishes all true canonical groups in the noiseless-signature case.
\end{theorem}

For a fixed pair, the failure probability is at most $(1-p)^m\le e^{-pm}$; a union bound over $O(N^2)$ pairs proves the result (Appendix~\ref{app:canon_proof}). The theorem is deliberately conditional: TV, Wasserstein, or mean separation imply anchor visibility only with assumptions on the anchor distribution, support overlap, covariance conditioning, and log-density boundedness. We therefore tune anchor size empirically.

If comparable Gaussian parameters are available in a shared coordinate system, nearest-mean, Hungarian, or OT matching are legitimate anchorless baselines. The next result instead covers a stricter unordered-label oracle model with no shared parameters or public reference evaluations.

\begin{proposition}[Anchorless non-identifiability for label-sensitive replay]
\label{prop:barrier}
Let there be $K\ge2$ replay modes, and suppose each client reports an unordered list of $K$ local symbols. For each client $i$, the local symbol order is related to the global mode order by an unknown permutation $\sigma_i\in S_K$. The server receives only the unordered local symbols and their within-client replay targets, not shared-coordinate parameters, public anchor evaluations, validation examples, or a shared random seed. If the downstream replay loss is label-sensitive---meaning that assigning a target from mode $k$ to mode $k'\ne k$ incurs error at least $c>0$---then any deterministic anchorless server must either represent all $K!$ possible global alignments for at least one client, select a single alignment and be wrong on at least a $1-1/K!$ fraction of equally likely permutations, or use additional side information not present in the oracle model. In the second case, the expected label-sensitive replay error is at least $c(1-1/K!)$ under the uniform prior over permutations.
\end{proposition}

The proof is an invariance argument: all $K!$ permutations induce the same observation, so a deterministic algorithm cannot identify the true alignment without extra information (Appendix~\ref{app:barrier_proof}). This is an identifiability statement, not a DP lower bound; coordinate-space Hungarian and OT baselines escape this oracle regime and are evaluated below.

We next account for replay release using R\'enyi DP composition \citep{mironov2017renyi}. Each round composes a local DP list learner with noisy aggregation of clipped per-mode summaries. Anchor matching and replay construction operate only on released objects and are post-processing.

\begin{theorem}[Replay-release user-level DP under RDP composition]
\label{thm:privacy}
Fix a communication round $t$. Let $\mathcal M^{\mathrm{list}}_t$ denote the local DP list-release mechanism run by participating clients, and let $\mathcal M^{\mathrm{agg}}_t$ denote the server-side noisy aggregation mechanism after canonical matching. Suppose $\mathcal M^{\mathrm{list}}_t$ satisfies R\'enyi DP of order $\alpha>1$ with cost $\eps^{\mathrm{list}}_t(\alpha)$. For the aggregation step, suppose each participating client contributes clipped sufficient statistics with $\ell_2$ sensitivity at most $S_\mu$ for means, Frobenius sensitivity at most $S_M$ for second moments, sensitivity at most $S_\pi$ for mixture weights, and sensitivity at most $S_y$ for replay-target summaries. Releasing noisy aggregates with Gaussian noise standard deviations $\sigma_\mu S_\mu$, $\sigma_M S_M$, $\sigma_\pi S_\pi$, and $\sigma_y S_y$ satisfies R\'enyi DP cost
\begin{equation}
\eps^{\mathrm{agg}}_t(\alpha)
\le
\frac{\alpha}{2\sigma_\mu^2}
+
\frac{\alpha}{2\sigma_M^2}
+
\frac{\alpha}{2\sigma_\pi^2}
+
\frac{\alpha}{2\sigma_y^2},
\end{equation}
up to standard subsampling amplification when only a random subset of clients participates. The replay-release mechanism in round $t$ therefore satisfies
\begin{equation}
\eps^{\mathrm{replay}}_t(\alpha)
\le
\eps^{\mathrm{list}}_t(\alpha)
+
\eps^{\mathrm{agg}}_t(\alpha).
\end{equation}
Over $T$ adaptive rounds, RDP costs compose as
\begin{equation}
\eps^{\mathrm{replay}}_{1:T}(\alpha)
\le
\sum_{t=1}^T
\eps^{\mathrm{replay}}_t(\alpha).
\end{equation}
The replay-release mechanism is $(\eps,\delta)$-DP for any $\delta>0$ with
\begin{equation}
\eps
=
\min_{\alpha>1}
\left\{
\eps^{\mathrm{replay}}_{1:T}(\alpha)
+
\frac{\log(1/\delta)}{\alpha-1}
\right\}.
\end{equation}
Deterministic anchor matching, canonical ID assignment, PSD projection, replay sampling from the released $Q_t$, and downstream use of the released replay samples are post-processing of the replay release and do not increase this replay-release privacy cost.
\end{theorem}

Theorem~\ref{thm:privacy} covers replay release only. End-to-end private training additionally requires a DP optimizer for communicated task-head updates, with budgets composed sequentially. Our tables report the replay-release budget unless stated otherwise. The proof is in Appendix~\ref{app:privacy_proof}.

Finally, we record a retention bound that connects replay quality to forgetting. The role of the bound is structural rather than quantitative: it shows that better replay (in TV to the ideal past-task distribution) directly controls forgetting up to optimization, sampling, and finite-client effects.

\begin{theorem}[Forgetting control via replay quality]
\label{thm:retention}
Let $\theta_t$ be the model after round $t$, trained with replay from $Q_t$. Let $Q_t^\star$ denote the ideal replay distribution over past-task embeddings and targets, and suppose $\mathrm{TV}(Q_t,Q_t^\star)\le\alpha_t$. Assume the replay loss $\ell(\theta;z,y)$ is bounded in $[0,1]$ and that the replay-loss class induced by the task head has uniform estimation error $O(\sqrt{d/n_r})$ from $n_r$ replay samples. Assume further that, for each past task $t$, excess task risk is controlled by excess ideal replay loss up to optimization and finite-client error:
\begin{equation}
R_t(\theta_T)-R_t(\theta_t)
\le
c_0\!\left(
\mathcal L_{Q_t^\star}(\theta_T)-\mathcal L_{Q_t^\star}(\theta_t)
\right)
+\xi_t
+O(C_{\min}^{-1/2}),
\end{equation}
for a universal constant $c_0>0$, where
\[
C_{\min}=\min_t |\cC_t|.
\]
If replay-augmented training achieves empirical replay optimization error at most $\xi_t$ at round $t$, then
\begin{equation}
\label{eq:forgetting_bound}
\bar F_T
=
\frac{1}{T}\sum_{t=1}^T
\bigl(R_t(\theta_T)-R_t(\theta_t)\bigr)
\le
O\!\left(
\bar\alpha+
\sqrt{\frac{d}{n_r}}+
\bar\xi+
C_{\min}^{-1/2}
\right),
\end{equation}
where
\[
\bar\alpha=\frac{1}{T}\sum_{t=1}^T\alpha_t,
\qquad
\bar\xi=\frac{1}{T}\sum_{t=1}^T\xi_t .
\]
\end{theorem}

The proof is in Appendix~\ref{app:retention_proof}. The bound is a stability statement, not an end-to-end convergence theorem. It motivates canonicalization: avoiding mode mismatch reduces the replay gap $\alpha_t$.

The bound is a stability statement, not an end-to-end convergence theorem. It motivates canonicalization: avoiding mode mismatch reduces the replay gap $\alpha_t$.

%----------------------------------------------------------------------
\section{Experiments}
\label{sec:experiments}
%----------------------------------------------------------------------

\subsection{Setup}

We evaluate CSLR on four federated continual NLP benchmarks, summarized in Table~\ref{tab:protocol}. Split-AG-News uses four AG News topics as sequential classification tasks with 20 clients and a Dirichlet$(0.5)$ partition. Split-Amazon uses five Amazon product domains with 50 clients and Dirichlet$(0.3)$ heterogeneity. OntoNotes-Domain evaluates sequential NER across six OntoNotes domains with 30 domain-stratified clients. MultiWOZ-Intent evaluates dialogue intent classification across seven service domains with 40 clients and Dirichlet$(0.5)$ partitioning. Classification tasks use frozen Sentence-BERT embeddings (all-MiniLM-L6-v2, $d=384$) with a trainable two-layer MLP head; NER uses frozen BERT-base token embeddings. We focus on continual adaptation of language-model heads over frozen representations, which isolates replay-distribution learning and keeps user-level accounting tractable; jointly trained encoders would require re-anchoring and are left to future work. Additional implementation details and statistical reporting conventions are given in Appendices~\ref{app:exp_details} and~\ref{app:stats}.

\begin{table}[t]
\centering
\small
\caption{Benchmark protocol summary. Task-ID availability follows the standard task-incremental setting.}
\label{tab:protocol}
\setlength{\tabcolsep}{3pt}
\begin{tabular}{lcccc}
\toprule
Benchmark & Tasks & Clients & Partition & Metric \\
\midrule
Split-AG-News    & 4 & 20 & Dir$(0.5)$ & Accuracy \\
Split-Amazon     & 5 & 50 & Dir$(0.3)$ & Accuracy \\
OntoNotes-Domain & 6 & 30 & Domain     & Entity F1 \\
MultiWOZ-Intent  & 7 & 40 & Dir$(0.5)$ & Accuracy \\
\bottomrule
\end{tabular}
\end{table}

Privacy is enforced for replay release with user-level $(\eps,\delta)$-DP for $\eps\in\{1,4,8\}$ and $\delta=10^{-5}$. The accountant composes the local DP list release with noisy aggregation of clipped replay summaries under R\'enyi DP \citep{mironov2017renyi}, using client participation rate $q=0.30$ (Appendix~\ref{app:accountant}). Anchor signatures, canonical matching, PSD projection, and replay sampling are post-processing. Unless stated otherwise, reported $\eps$ values refer to replay release only, not to a full deployed budget that also privatizes task-head updates. Results are averaged over five seeds and reported as mean $\pm$ standard deviation, with summary-level confidence intervals for effect-size comparisons.

We report average accuracy $\mathrm{AA}=\tfrac{1}{T}\sum_j a_{T,j}$, backward transfer $\mathrm{BWT}=\tfrac{1}{T-1}\sum_{j<T}(a_{T,j}-a_{j,j})$, and forward transfer $\mathrm{FWT}=\tfrac{1}{T-1}\sum_{j\ge2}(a_{j-1,j}-b_j)$, where $a_{r,j}$ is performance on task $j$ after round $r$ and $b_j$ is the zero-shot baseline. For NER, $a_{r,j}$ is entity-level micro-F1.

Baselines test CSLR both as a DP-FCL method and as an anchoring mechanism. We compare with DP-SGD without replay, DP-Stat-Share \citep{yoon2021federated}, FedCIL+DP \citep{qi2023better}, and non-private FedTA$^\dagger$ \citep{yang2024fedta} as an upper-bound reference. To isolate the role of anchors, we also evaluate CSLR-NoAnchor, CSLR-NearestMean, CSLR-Hungarian, and CSLR-OptimalTransport. NearestMean, Hungarian, and OT operate on shared-coordinate component parameters, making them strong anchorless baselines for cases where direct parameter matching is meaningful.

\subsection{Main Results}

\begin{table*}[t]
\centering
\small
\caption{Final average task metric after all tasks, mean $\pm$ standard deviation over five seeds.
Accuracy is reported for AG-News, Amazon, and MultiWOZ; entity-level micro-F1 is reported for OntoNotes.}
\label{tab:main}
\begin{tabular}{lcccc}
\toprule
Method & AG-News & Amazon & OntoNotes & MultiWOZ \\
\midrule
DP-SGD (no replay)        & $51.6 \pm 0.5$ & $42.0 \pm 0.6$ & $37.7 \pm 0.7$ & $44.2 \pm 0.6$ \\
DP-Stat-Share             & $57.4 \pm 0.4$ & $48.6 \pm 0.5$ & $43.0 \pm 0.5$ & $49.9 \pm 0.4$ \\
FedCIL + DP               & $55.0 \pm 0.6$ & $46.2 \pm 0.7$ & $40.6 \pm 0.8$ & $47.4 \pm 0.6$ \\
FedTA$^\dagger$           & $63.5 \pm 0.3$ & $56.4 \pm 0.4$ & $50.8 \pm 0.3$ & $56.0 \pm 0.4$ \\
\midrule
CSLR-NoAnchor (Random)    & $54.2 \pm 0.7$ & $44.5 \pm 0.8$ & $40.3 \pm 0.9$ & $46.6 \pm 0.7$ \\
CSLR-NearestMean          & $56.8 \pm 0.5$ & $47.1 \pm 0.6$ & $42.4 \pm 0.6$ & $48.9 \pm 0.5$ \\
CSLR-Hungarian            & $58.1 \pm 0.4$ & $49.3 \pm 0.5$ & $44.1 \pm 0.5$ & $50.2 \pm 0.4$ \\
CSLR-OptimalTransport     & $58.5 \pm 0.4$ & $49.8 \pm 0.4$ & $44.7 \pm 0.5$ & $50.8 \pm 0.4$ \\
\textbf{CSLR (ours)}      & $\mathbf{62.1 \pm 0.3}$ & $\mathbf{53.8 \pm 0.4}$ & $\mathbf{48.6 \pm 0.4}$ & $\mathbf{53.8 \pm 0.3}$ \\
\bottomrule
\end{tabular}
\end{table*}

\begin{table*}[t]
\centering
\small
\caption{Backward transfer (BWT, \%; closer to zero is better) at $\eps = 4$ across five seeds.}
\label{tab:bwt}
\begin{tabular}{lcccc}
\toprule
Method & AG-News & Amazon & OntoNotes & MultiWOZ \\
\midrule
DP-SGD (no replay)        & $-18.4 \pm 0.6$ & $-22.3 \pm 0.7$ & $-24.9 \pm 0.8$ & $-20.6 \pm 0.6$ \\
DP-Stat-Share             & $-11.2 \pm 0.4$ & $-15.8 \pm 0.5$ & $-17.5 \pm 0.6$ & $-15.0 \pm 0.5$ \\
FedCIL + DP               & $-13.5 \pm 0.5$ & $-18.0 \pm 0.6$ & $-19.3 \pm 0.7$ & $-17.2 \pm 0.6$ \\
\midrule
CSLR-NoAnchor (Random)    & $-14.8 \pm 0.7$ & $-19.5 \pm 0.8$ & $-20.7 \pm 0.9$ & $-17.7 \pm 0.7$ \\
CSLR-NearestMean          & $-11.9 \pm 0.5$ & $-16.2 \pm 0.6$ & $-17.9 \pm 0.6$ & $-15.4 \pm 0.5$ \\
CSLR-Hungarian            & $-10.1 \pm 0.4$ & $-14.0 \pm 0.5$ & $-15.6 \pm 0.5$ & $-13.1 \pm 0.4$ \\
CSLR-OptimalTransport     & $-9.6 \pm 0.4$ & $-13.5 \pm 0.4$ & $-15.1 \pm 0.5$ & $-12.6 \pm 0.4$ \\
\textbf{CSLR (ours)}      & $\mathbf{-7.1 \pm 0.3}$ & $\mathbf{-10.6 \pm 0.4}$ & $\mathbf{-12.0 \pm 0.4}$ & $\mathbf{-9.9 \pm 0.3}$ \\
\bottomrule
\end{tabular}
\end{table*}

Table~\ref{tab:main} shows that CSLR improves AA by 3.9--5.6 points over the strongest non-CSLR DP baseline at $\eps=4$ and remains within range of non-private FedTA on most tasks. Relative to the strongest coordinate-space matcher, CSLR-OptimalTransport,
CSLR keeps a 3.0--4.0 point margin. Relative to CSLR-Hungarian, the margin is
3.6--4.5 points. Table~\ref{tab:bwt} shows the same ordering for forgetting. Effect-size CIs and FWT are in Appendix~\ref{app:more_results}.

\subsection{Ablation Studies}

The matcher ablation asks whether CSLR is merely better coordinate-space matching. Hungarian and OT improve over random assignment, but the residual gap to CSLR is consistent. Direct matchers compare clipped noisy parameter vectors, whereas anchor signatures compare candidates in a lower-dimensional public reference space.

\begin{figure}[t]
 \centering
 \begin{tikzpicture}
 \begin{axis}[
   width=0.9\columnwidth, height=4.5cm,
   xlabel={Anchor set size $m$},
   ylabel={Average accuracy (\%)},
   xmin=0, xmax=210,
   ymin=45, ymax=68,
   xtick={10,25,50,100,200},
   grid=major,
   grid style={dashed,gray!30},
   mark size=2pt,
   thick,
 ]
 \addplot[blue, mark=*, thick] coordinates {
   (10, 48.3) (25, 56.7) (50, 62.1) (100, 62.8) (200, 63.1)
 };
 \draw[red, dashed, thick] (axis cs:45,45) -- (axis cs:45,68) node[above,font=\footnotesize] {$m^\star$};
 \end{axis}
 \end{tikzpicture}
 \caption{Average accuracy versus anchor set size $m$ on Split-AG-News ($\eps = 4$). The dashed line marks the empirical saturation point $m^\star\approx 50$ for this encoder, anchor distribution, and task stream.}
 \label{fig:anchor_sweep}
\end{figure}
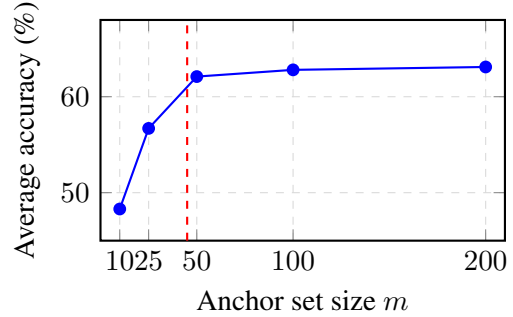

Figure~\ref{fig:anchor_sweep} shows that performance rises up to about $m=50$ anchors and then saturates. This is an empirical saturation point for the chosen encoder, anchor distribution, and task stream, not a direct numerical consequence of Theorem~\ref{thm:canon}. Appendix~\ref{app:more_results} reports matching signature-margin diagnostics.

Table~\ref{tab:privacy_scope} reports privacy--utility trends under the
privacy scope used by each method. The replay-based methods are compared at
matched replay-release budgets. DP-SGD is included as an optimizer-private
no-replay reference, but its $\eps$ is not directly budget-equivalent to the
replay-release $\eps$ used for CSLR, DP-Stat-Share, and FedCIL+DP. End-to-end
private deployment would require composing the replay-release accountant with
a private optimizer for communicated task-head updates.

\subsection{Qualitative Analysis}

The matcher ablation in Table~\ref{tab:main} suggests that canonical anchor signatures provide more reliable cross-client alignment than uncanonicalized client-local orderings.

%----------------------------------------------------------------------
\section{Discussion and Conclusion}
\label{sec:discussion_conclusion}

We introduced Canonicalized Stable-List Replay, a replay-release mechanism for private federated continual learning of language-model heads over frozen representations. CSLR uses stable private lists because, in some mixture settings, releasing a small set of candidate replay distributions can be easier than privately selecting one final estimate \citep{afzali2024agnostic}. Anchor canonicalization then makes these unordered client lists comparable, resolving the permutation ambiguity that otherwise prevents meaningful cross-client aggregation.

The anchors are shared reference points in embedding space, not additional replay data. They may come from public sentences, synthetic public embeddings, or fixed random projections, as long as the anchor set is fixed before private release. After the DP list learner releases candidate distributions, anchor signatures, matching, simplex projection of mixture weights, PSD projection, and replay sampling are all post-processing and add no replay-release privacy cost.

CSLR is limited by its frozen-encoder assumption, by the expressiveness of Gaussian replay components, and by the scope of the non-identifiability result, which applies to an unordered-label oracle model rather than all realistic matching pipelines. Shared-coordinate anchorless matchers therefore remain important baselines, as reflected in our nearest-mean, Hungarian, and optimal-transport comparisons. Overall, CSLR shows that stable private lists, anchor-based canonicalization, and DP aggregation can improve retention under replay-release user-level DP, while end-to-end private deployment still requires composition with a private optimizer for communicated task-head updates. Richer replay families and encoder-update regimes are natural directions for future work.

%----------------------------------------------------------------------
\section*{Limitations}
%----------------------------------------------------------------------

Our approach assumes a high-quality frozen pretrained sentence encoder; for languages or domains where such encoders are unavailable, the embedding-space density modeling may perform poorly. The non-identifiability result is proven only for the abstract unlabeled-oracle model, and extending the analysis to broader distribution classes and realistic matching pipelines remains open. Because the framing operates over frozen representations, adaptation targets downstream classification or extraction heads rather than full-parameter fine-tuning; if representation shifts are necessary, embeddings warp sequentially, breaking baseline structural metrics and requiring regular re-anchoring strategies. User-level DP with small $\eps$ remains challenging, and our improvements still show notable gaps to non-private references. Finally, the formal privacy guarantee in Theorem~\ref{thm:privacy} covers the replay-release mechanism: local private list construction, noisy aggregation of replay summaries, and post-processing into a replay generator. It does not by itself cover communicated task-head updates. A deployed end-to-end private FCL system must therefore combine CSLR with a private optimizer for model updates, secure aggregation, and communication-channel protections, with the final privacy budget obtained by composing all privacy-consuming components.

%----------------------------------------------------------------------
\section*{Ethical Considerations}
%----------------------------------------------------------------------

This work addresses privacy-preserving federated continual learning of language-model heads over frozen representations. CSLR is designed to reduce the need for raw-text replay by releasing private replay distributions and target summaries under a replay-release DP accountant. The formal privacy guarantee applies to that replay-release mechanism specifically. End-to-end privacy in deployment additionally requires private model-update training, secure aggregation, and communication-channel protections. Anchor sentences should be drawn from corpora licensed for the intended use, and operators should verify that public anchors do not encode sensitive or systematically biased attributes that could shift downstream representations in unintended ways.

\bibliography{references}

\appendix

%----------------------------------------------------------------------
\section{Full Proof of Theorem~\ref{thm:canon}}
\label{app:canon_proof}
%----------------------------------------------------------------------

\begin{proof}
Fix a pair $(q,q')$ that should receive different canonical IDs. By Definition~\ref{def:anchorvisible}, the event
\begin{equation}
E_{q,q'}=\{u:|\log q(u)-\log q'(u)|\ge \gamma\}
\end{equation}
has anchor probability at least $p$. Since anchors are drawn independently from $\rho$, the probability that no anchor lands in $E_{q,q'}$ is
\begin{equation}
\Pr[\forall j,\ \phi(a_j)\notin E_{q,q'}] \le (1-p)^m \le \exp(-pm).
\end{equation}
Let $\mathcal P$ be the set of all distinguishable pairs among the $N$ candidate list elements. Since $|\mathcal P|\le N^2/2$, the probability that any distinguishable pair is missed is at most $|\mathcal P|e^{-pm}\le \tfrac{N^2}{2}e^{-pm}$. If $m\ge p^{-1}\log(N^2/\eta)$, this failure probability is at most $\eta$. On the complement of this failure event, every pair that should be separated differs by at least $\gamma$ in at least one signature coordinate, and any noiseless matching rule that separates signatures at threshold $\gamma/2$ cannot merge two different canonical groups. This proves the theorem.
\end{proof}

%----------------------------------------------------------------------
\section{Proof of Proposition~\ref{prop:barrier}}
\label{app:barrier_proof}
%----------------------------------------------------------------------

\begin{proof}
The server's observation is invariant to relabeling by any permutation in $S_K$. More precisely, for any client $i$ and any two permutations $\sigma,\sigma'\in S_K$, applying $\sigma$ or $\sigma'$ to the client's local symbols yields the same unordered multiset observed by the server. No deterministic function of that observation can therefore distinguish which of the $K!$ global alignments is the correct one. An algorithm that maintains correctness for all possible alignments must represent all $K!$ hypotheses; if instead it selects one alignment, then under the uniform prior over $S_K$ it is correct with probability $1/K!$ and incorrect otherwise. Since the loss is label-sensitive and every incorrect target assignment incurs at least $c$, the expected replay-target error is at least $c(1-1/K!)$. Any procedure that does better must use information excluded from the oracle model: shared-coordinate parameters, public anchors, validation data, or a common random seed.
\end{proof}

%----------------------------------------------------------------------
\section{Proof of Theorem~\ref{thm:privacy}}
\label{app:privacy_proof}
%----------------------------------------------------------------------

\begin{proof}
The replay-release mechanism is the adaptive composition of the local list-release mechanism and the server aggregation mechanism. By assumption, the local list learner at round $t$ satisfies R\'enyi DP cost $\eps^{\mathrm{list}}_t(\alpha)$. For the aggregation step, each clipped query is released with a Gaussian mechanism. A clipped vector-valued query with sensitivity $S$ and Gaussian noise $\mathcal N(0,\sigma^2S^2I)$ satisfies R\'enyi DP of order $\alpha>1$ with cost $\alpha/(2\sigma^2)$ \citep{mironov2017renyi}. Applying this bound to the mean, second-moment, mixture-weight, and replay-target releases and adding the costs gives $\eps^{\mathrm{agg}}_t(\alpha)$. Sequential composition within the round gives $\eps^{\mathrm{replay}}_t(\alpha)\le \eps^{\mathrm{list}}_t(\alpha)+\eps^{\mathrm{agg}}_t(\alpha)$, and adaptive composition across rounds gives $\eps^{\mathrm{replay}}_{1:T}(\alpha)\le\sum_t\eps^{\mathrm{replay}}_t(\alpha)$. The standard conversion from RDP to approximate DP yields $(\eps,\delta)$ with $\eps=\eps^{\mathrm{replay}}_{1:T}(\alpha)+\log(1/\delta)/(\alpha-1)$ for each fixed $\alpha$, and minimizing over $\alpha>1$ gives the stated expression. Anchor matching, canonical assignment, PSD projection, and replay sampling are deterministic or randomized functions of released DP objects, so they are post-processing and do not increase privacy loss.
\end{proof}

%----------------------------------------------------------------------
\section{Proof of Theorem~\ref{thm:retention}}
\label{app:retention_proof}
%----------------------------------------------------------------------

\begin{proof}
Fix a past task $t$. Let
\[
\mathcal L_Q(\theta)=\mathbb E_{(z,y)\sim Q}[\ell(\theta;z,y)]
\]
denote the replay loss under distribution $Q$. Because $\ell$ is bounded in $[0,1]$ and $\mathrm{TV}(Q_t,Q_t^\star)\le\alpha_t$, the variational characterization of total variation gives, for every $\theta$,
\begin{equation}
\left|
\mathcal L_{Q_t}(\theta)-\mathcal L_{Q_t^\star}(\theta)
\right|
\le
\alpha_t .
\end{equation}
Therefore, for any two parameter values $\theta$ and $\theta'$,
\begin{equation}
\mathcal L_{Q_t^\star}(\theta)-\mathcal L_{Q_t^\star}(\theta')
\le
\mathcal L_{Q_t}(\theta)-\mathcal L_{Q_t}(\theta')
+2\alpha_t .
\label{eq:tv_transfer}
\end{equation}

Let $\widehat Q_t$ be the empirical replay distribution formed from $n_r$ samples from $Q_t$. By the assumed uniform estimation bound for the replay-loss class,
\begin{equation}
\sup_{\theta}
\left|
\mathcal L_{Q_t}(\theta)-\mathcal L_{\widehat Q_t}(\theta)
\right|
\le
O\!\left(\sqrt{\frac{d}{n_r}}\right)
\end{equation}
with the stated probability. Hence, for any $\theta,\theta'$,
\begin{equation}
\mathcal L_{Q_t}(\theta)-\mathcal L_{Q_t}(\theta')
\le
\mathcal L_{\widehat Q_t}(\theta)-\mathcal L_{\widehat Q_t}(\theta')
+
O\!\left(\sqrt{\frac{d}{n_r}}\right).
\label{eq:uniform_transfer}
\end{equation}

By the replay optimization-error assumption, the replay-augmented training procedure produces $\theta_T$ such that its empirical replay loss on past-task replay is within $\xi_t$ of the comparator $\theta_t$:
\begin{equation}
\mathcal L_{\widehat Q_t}(\theta_T)
-
\mathcal L_{\widehat Q_t}(\theta_t)
\le
\xi_t .
\label{eq:opt_error}
\end{equation}
Combining \eqref{eq:tv_transfer}, \eqref{eq:uniform_transfer}, and \eqref{eq:opt_error} with $\theta=\theta_T$ and $\theta'=\theta_t$ yields
\begin{equation}
\mathcal L_{Q_t^\star}(\theta_T)
-
\mathcal L_{Q_t^\star}(\theta_t)
\le
O\!\left(
\alpha_t+
\sqrt{\frac{d}{n_r}}+
\xi_t
\right).
\end{equation}

The assumed risk-control condition then gives
\begin{equation}
R_t(\theta_T)-R_t(\theta_t)
\le
O\!\left(
\alpha_t+
\sqrt{\frac{d}{n_r}}+
\xi_t+
C_{\min}^{-1/2}
\right).
\end{equation}
Averaging this inequality over $t=1,\ldots,T$ gives
\begin{equation}
\bar F_T
=
\frac{1}{T}\sum_{t=1}^T
\bigl(R_t(\theta_T)-R_t(\theta_t)\bigr)
\le
O\!\left(
\bar\alpha+
\sqrt{\frac{d}{n_r}}+
\bar\xi+
C_{\min}^{-1/2}
\right),
\end{equation}
where $\bar\alpha=T^{-1}\sum_t\alpha_t$ and $\bar\xi=T^{-1}\sum_t\xi_t$. This proves the claim.
\end{proof}

%----------------------------------------------------------------------
\section{Privacy-Accountant Details}
\label{app:accountant}
%----------------------------------------------------------------------

\begin{table*}[t]
\centering
\small
\caption{
Split-AG-News privacy--utility comparison under distinct privacy scopes.
For replay-based methods, reported $\eps$ denotes the replay-release budget only:
local replay-summary release plus noisy replay-summary aggregation.
For DP-SGD, reported $\eps$ denotes the optimizer-update budget, since no replay
object is released. Thus DP-SGD is an optimizer-private reference, not a
budget-equivalent replay-release baseline. All values are average accuracy
(\%), mean $\pm$ standard deviation over five seeds; $\delta=10^{-5}$.
}
\label{tab:privacy_scope}
\resizebox{\textwidth}{!}{%
\begin{tabular}{@{}llccccc@{}}
\toprule
Method
& Privacy scope of reported $\eps$
& Replay release privatized?
& Task-head updates privatized?
& $\eps=1$
& $\eps=4$
& $\eps=8$ \\
\midrule
DP-SGD (no replay)
& Optimizer-update DP only
& No
& Yes
& $43.4 \pm 0.6$
& $51.6 \pm 0.5$
& $56.1 \pm 0.5$ \\

DP-Stat-Share
& Replay-release DP only
& Yes
& No
& $48.6 \pm 0.4$
& $57.4 \pm 0.4$
& $60.8 \pm 0.4$ \\

FedCIL + DP
& Replay-release DP only
& Yes
& No
& $46.1 \pm 0.5$
& $55.0 \pm 0.6$
& $58.5 \pm 0.5$ \\

\textbf{CSLR}
& Replay-release DP only
& Yes
& No
& $\mathbf{54.5 \pm 0.4}$
& $\mathbf{62.1 \pm 0.3}$
& $\mathbf{64.4 \pm 0.3}$ \\
\bottomrule
\end{tabular}%
}
\end{table*}

Table~\ref{tab:accountant} reports the replay-release accountant used for the reported privacy budgets. The accountant includes both privacy-consuming replay-release stages: the local DP list learner and the noisy aggregation of clipped canonical replay summaries. All component means, second moments, mixture weights, and replay-target summaries are clipped before release or aggregation. The listed noise multipliers are calibrated by R\'enyi DP composition \citep{mironov2017renyi} under client participation rate $q=0.30$ and target $\delta=10^{-5}$. Deterministic anchor-signature computation, canonical matching, PSD projection, and replay sampling are post-processing of released private objects and therefore incur no additional privacy loss. For a deployed end-to-end system, these replay-release costs must be composed with the privacy cost of the optimizer used to communicate task-head updates.

\begin{table}[t]
\centering
\small
\caption{Replay-release privacy-accounting settings used for the reported privacy budgets. Noise multipliers are calibrated by RDP composition with client participation rate $q=0.30$ and $\delta=10^{-5}$. Clipping radii apply to component means, second moments, mixture weights, and target summaries before secure aggregation.}
\label{tab:accountant}
\begin{tabular}{lccc}
\toprule
Target $\eps$ & 1 & 4 & 8 \\
\midrule
$\delta$ & $10^{-5}$ & $10^{-5}$ & $10^{-5}$ \\
Client sampling $q$ & 0.30 & 0.30 & 0.30 \\
$B_\mu,B_M,B_\pi,B_y$ & 1.0 & 1.0 & 1.0 \\
Noise multiplier $\sigma$ & 3.20 & 1.45 & 0.95 \\
RDP orders searched & \multicolumn{3}{c}{$\{2,3,\ldots,64\}$} \\
\bottomrule
\end{tabular}
\end{table}

%----------------------------------------------------------------------
\section{Experimental Details}
\label{app:exp_details}
%----------------------------------------------------------------------

The sentence encoder is all-MiniLM-L6-v2 (384-dim, frozen), and the task head is a 2-layer MLP ($384\to 128\to C$) with ReLU activations. The list size is $L=4$ and the number of canonical modes $K$ is set to the number of tasks. The anchor set size is $m=100$ Wikipedia sentences. The replay-release DP mechanism composes the local DP list learner with Gaussian mechanisms for noisy canonical aggregation, analyzed by R\'enyi DP accounting with $\delta=10^{-5}$. Each round uses $n_r=500$ replay samples, learning rate $3\times 10^{-4}$ with Adam, 50 rounds per task and $50T$ rounds total, and a client participation rate of 30\% per round.

The five independent seeds are 13, 17, 19, 23, and 29. Dataset indices, client partitions, anchor dictionaries, model initializations, and client-sampling schedules are all generated deterministically from these seed IDs. The reference dictionary uses static Wikipedia sentences extracted once per seed and then held fixed across all methods within that seed. Nearest-mean matching, Hungarian matching, and optimal-transport matching all operate directly over the privacy-modified parametric summaries. The anonymous experiment artifact is organized around these seed IDs and contains raw per-seed traces, configuration files, and table-generation scripts for regenerating the main tables. Each client fits candidate components in the frozen embedding space after clipping embeddings to radius $B_z$. The local primitive is a private list learner, instantiated here by repeated clipped EM initializations followed by noisy sufficient-statistic release. Candidates with mixture weight below $\pi_{\min}$ are retained in the local list but downweighted during server aggregation rather than discarded, which avoids data-dependent list-size leakage.

For OntoNotes-Domain, token embeddings are grouped by predicted or gold entity class during local training, and replay samples are token-level embeddings paired with clipped BIO tag-frequency targets. During evaluation, the task head decodes BIO tags under a deterministic transition mask that prevents invalid I-tags after O or incompatible entity types. The reported NER performance is entity-level micro-F1.

For the OT baseline, each round solves
\begin{equation}
\min_{P\in\Pi(w,\bar w)}
\sum_{\ell,k}P_{\ell k}c_{\ell k}
+\lambda_{\mathrm{OT}}\sum_{\ell,k}P_{\ell k}\log P_{\ell k},
\end{equation}
with ground cost $c_{\ell k}=\|\mu_{\ell}-\bar\mu_k\|_2^2+\tau\|\Sigma_{\ell}-\bar\Sigma_k\|_F^2$. We convert the resulting coupling to a hard assignment by choosing the maximum-coupling global mode for each local component, with ties resolved deterministically by mode index.

%----------------------------------------------------------------------
\section{Statistical Reporting}
\label{app:stats}
%----------------------------------------------------------------------

For the AA comparison in Table~\ref{tab:effect_sizes}, we use the reported means $\bar x,\bar y$ and standard deviations $s_x,s_y$ over $n=5$ seeds. The displayed interval for $\Delta=\bar x-\bar y$ uses
\begin{equation}
\Delta \pm t_{0.975,4}\sqrt{\frac{s_x^2}{5}+\frac{s_y^2}{5}},
\end{equation}
with $t_{0.975,4}=2.776$. This is a summary-level interval and is conservative relative to a paired-seed interval when methods share the same partitions and client-sampling schedules. The artifact traces should be used for paired bootstrap or paired $t$-tests when the anonymous artifact is evaluated.

%----------------------------------------------------------------------
\section{Additional Empirical Results}
\label{app:more_results}
%----------------------------------------------------------------------

This section collects the auxiliary tables referenced from the main body. Table~\ref{tab:effect_sizes} reports AA confidence intervals. Table~\ref{tab:fwt} reports forward transfer at $\eps=4$ across five seeds; it confirms that the relative ordering by FWT is consistent with the AA and BWT orderings reported in the main tables. Table~\ref{tab:anchor_diag} reports the anchor diagnostic on Split-AG-News supporting the saturation observed in Figure~\ref{fig:anchor_sweep}: the empirical hit rate $\hat p_\gamma$ and the minimum pairwise signature margin both plateau in the same regime as the accuracy curve. Table~\ref{tab:list_size} reports the list-size ablation referenced from the ablation studies.

\begin{table}[t]
\centering
\small
\caption{Effect-size summary for final AA at $\eps=4$. Intervals are approximate 95\% confidence intervals for the mean difference, computed from the reported five-seed standard deviations using Welch-style standard errors. Positive intervals indicate CSLR improves over the comparator.}
\label{tab:effect_sizes}
\setlength{\tabcolsep}{4pt}
\begin{tabular}{llcc}
\toprule
Dataset & Comparator & $\Delta$ AA & 95\% CI \\
\midrule
AG-News   & DP-Stat-Share & 4.7 & $[4.1,5.3]$ \\
Amazon    & DP-Stat-Share & 5.2 & $[4.4,6.0]$ \\
OntoNotes & DP-Stat-Share & 5.6 & $[4.8,6.4]$ \\
MultiWOZ  & DP-Stat-Share & 3.9 & $[3.3,4.5]$ \\
\midrule
AG-News   & CSLR-OT       & 3.6 & $[3.0,4.2]$ \\
Amazon    & CSLR-OT       & 4.0 & $[3.3,4.7]$ \\
OntoNotes & CSLR-OT       & 3.9 & $[3.1,4.7]$ \\
MultiWOZ  & CSLR-OT       & 3.0 & $[2.4,3.6]$ \\
\bottomrule
\end{tabular}
\end{table}

\begin{table}[t]
\centering
\scriptsize
\setlength{\tabcolsep}{3pt}
\caption{Forward transfer (FWT, \%; higher is better), mean $\pm$ standard deviation over five seeds at $\eps = 4$.}
\label{tab:fwt}
\begin{tabular}{lcccc}
\toprule
Method & AG-News & Amazon & OntoNotes & MultiWOZ \\
\midrule
DP-SGD          & $0.2{\pm}0.1$ & $0.1{\pm}0.1$ & $0.3{\pm}0.1$ & $0.2{\pm}0.1$ \\
DP-Stat-Share   & $1.4{\pm}0.2$ & $1.1{\pm}0.2$ & $1.6{\pm}0.3$ & $1.3{\pm}0.2$ \\
FedCIL+DP       & $0.8{\pm}0.2$ & $0.6{\pm}0.1$ & $0.9{\pm}0.2$ & $0.7{\pm}0.2$ \\
CSLR-OT         & $1.8{\pm}0.3$ & $1.5{\pm}0.2$ & $1.9{\pm}0.3$ & $1.6{\pm}0.2$ \\
\textbf{CSLR}   & $\mathbf{2.5{\pm}0.2}$ & $\mathbf{2.1{\pm}0.3}$ & $\mathbf{2.8{\pm}0.2}$ & $\mathbf{2.4{\pm}0.2}$ \\
\bottomrule
\end{tabular}
\end{table}

\begin{table}[t]
\centering
\small
\caption{Anchor diagnostic on Split-AG-News at $\eps=4$. The estimated hit rate $\hat p_{\gamma}$ is the fraction of sampled anchors with signature separation at least $\gamma$ for pairs assigned to different canonical modes.}
\label{tab:anchor_diag}
\begin{tabular}{lccccc}
\toprule
$m$ & 10 & 25 & 50 & 100 & 200 \\
\midrule
$\hat p_{\gamma}$ & 0.18 & 0.31 & 0.43 & 0.46 & 0.47 \\
Min. margin       & 0.07 & 0.11 & 0.18 & 0.19 & 0.20 \\
\bottomrule
\end{tabular}
\end{table}

\begin{table}[t]
\centering
\scriptsize
\setlength{\tabcolsep}{3pt}
\caption{List-size ablation on Split-AG-News at $\eps=4$, mean $\pm$ standard deviation over five seeds.}
\label{tab:list_size}
\begin{tabular}{lccccc}
\toprule
$L$ & 1 & 2 & 4 & 8 & 16 \\
\midrule
AA  & $57.4{\pm}0.4$ & $60.2{\pm}0.4$ & $62.1{\pm}0.3$ & $62.4{\pm}0.3$ & $62.5{\pm}0.4$ \\
BWT & $-11.2{\pm}0.4$ & $-8.9{\pm}0.4$ & $-7.1{\pm}0.3$ & $-6.9{\pm}0.3$ & $-6.8{\pm}0.4$ \\
\bottomrule
\end{tabular}
\end{table}

\end{document}